\documentclass[12pt]{iopart}

\usepackage[dvipdfmx]{graphicx}
\usepackage{subfigure}
\usepackage{algcompatible}
\usepackage[size=small]{caption}
\DeclareCaptionFormat{algorithm}{\vskip-15pt\hrulefill\par#1#2#3\vskip-6pt\hrulefill}
\usepackage{etoolbox}
\AtEndEnvironment{algorithm}{\noindent\hrulefill\par\nobreak\vskip-5pt}
\usepackage{newfloat}
\DeclareFloatingEnvironment[
  fileext=loa,
  listname=List of Algorithms,
  name=ALGORITHM,
  placement=tbhp,
]{algorithm}
\expandafter\let\csname equation*\endcsname\relax
\expandafter\let\csname endequation*\endcsname\relax
\usepackage{amsmath}
\usepackage{amssymb}
\usepackage{mathtools}

\DeclareMathOperator*{\argmax}{arg\,max}
\usepackage{amsthm}
\newtheorem{theorem}{Theorem}
\usepackage{xcolor}
\usepackage{diagbox}
\usepackage{multirow}
\usepackage{ifthen}
\newcommand*{\versionmath}{2}
\usepackage{braket}

\begin{document}


\title{Deterministic Quantum Annealing Expectation-Maximization Algorithm}

\author{Hideyuki Miyahara}
\ead{hideyuki\_miyahara@mist.i.u-tokyo.ac.jp}


\address{%
Graduate School of Information Science and Technology,
The University of Tokyo,
7-3-1 Hongosanchome Bunkyo-ku Tokyo 113-8656, Japan
}%


\author{Koji Tsumura}

\address{%
Graduate School of Information Science and Technology,
The University of Tokyo,
7-3-1 Hongosanchome Bunkyo-ku Tokyo 113-8656, Japan
}%

\author{Yuki Sughiyama}
\address{%
Institute of Industrial Science, The University of Tokyo,
4-6-1, Komaba, Meguro-ku, Tokyo 153-8505, Japan
}%


\date{\today}

\begin{abstract}
Maximum likelihood estimation (MLE) is one of the most important methods in machine learning, and the expectation-maximization (EM) algorithm is often used to obtain maximum likelihood estimates.
However, EM heavily depends on initial configurations and fails to find the global optimum.
On the other hand, in the field of physics, quantum annealing (QA) was proposed as a novel optimization approach.
Motivated by QA, we propose a quantum annealing extension of EM, which we call the deterministic quantum annealing expectation-maximization (DQAEM) algorithm.
We also discuss its advantage in terms of the path integral formulation.
Furthermore, by employing numerical simulations, we illustrate how it works in MLE and show that DQAEM outperforms EM.
\end{abstract}


\maketitle


\section{Introduction} \label{intro}

Machine learning gathers considerable attention in a wide range of fields~\cite{Bishop01, Murphy01}.
In unsupervised learning, which is a major branch of machine learning, maximum likelihood estimation (MLE) plays an important role to characterize a given data set.
One of the most common and practical approaches for MLE is the expectation-maximization (EM) algorithm.
Although EM is widely used~\cite{Dempster01}, it is also known to be trapped in local optima depending on initial configurations due to non-convexity of log-likelihood functions.

One of the breakthroughs for non-convex optimization is simulated annealing (SA), proposed by Kirkpatrick \textit{et al.}~\cite{Kirkpatrick01, Kirkpatrick02}.
In SA, a random variable, which mimics thermal fluctuations, is added during the optimization process to overcome potential barriers in non-convex optimization.
Moreover, the global convergence of SA is guaranteed when the annealing process is infinitely long \cite{Geman01}.
%
Motivated by SA and its variant~\cite{Rose01, Rose02}, Ueda and Nakano developed a deterministic simulated annealing expectation-maximization (DSAEM) algorithm~\footnote{This algorithm is originally called the deterministic annealing expectation-maximization algorithm in Ref.~\cite{Ueda01}. However, to distinguish our and their approaches, we refer to it as DSAEM in this paper.} by introducing thermal fluctuations into EM~\cite{Ueda01}.
This approach succeeded in improving the performance of EM without the increase of numerical costs.
However, problems caused by non-convexity are still remaining.

Another approach for non-convex optimization is quantum annealing (QA), which was proposed in Refs.~\cite{Apolloni01, Finnila01, Kadowaki01}, and it has been intensively studied by many physicists~\cite{Santoro01, Santoro02, Farhi01, Morita01, Brooke01, Falco01, Falco02, Das01, Marto01}.
One of the reasons why QA attracts great interest is that, in some cases, it outperforms SA~\cite{Morita01}.
Another reason is that QA can be directly implemented on a quantum computer, and much effort is devoted to realize one via many approaches~\cite{Lloyd01, Rosentrost01, Wiebe01, Johnson01}.
Moreover, quantum algorithms for machine learning are extensively studied~\cite{Schuld01, Aaronson01, Robentrost01}.


In this study, to improve the performances of EM and DSAEM, by employing QA we propose a new algorithm, which we call the deterministic quantum annealing expectation-maximization (DQAEM) algorithm.
To be more precise, by quantizing hidden variables in EM and adding a non-commutative term, we extend classical algorithms: EM and DSAEM.
Then, we discuss the mechanism of DQAEM from the viewpoint of the path integral formulation.
Through this discussion, we elucidate how DQAEM overcomes the problem of local optima.
Furthermore, as applications, we focus on clustering problems with Gaussian mixture models (GMMs).
Here, it is confirmed that DQAEM outperforms EM and DSAEM through numerical simulations.

This paper is organized as follows.
In Sec.~\ref{review-em}, we review MLE and EM to prepare for DQAEM.
In Sec.~\ref{dqaem-gmm-00}, which is the main section of this paper, we present the formulation of DQAEM.
Next, we give an interpretation of DQAEM and show its advantage over EM and DSAEM from the viewpoint of the path integral formulation in Sec.~\ref{path-dqaem-10}.
Then, in Sec.~\ref{sec-numerical-00}, we introduce GMMs and demonstrate numerical simulations.
Here, it is found that DQAEM is superior to EM and DSAEM.
Finally, Sec.~\ref{conc} conclude this paper.

\section{Maximum likelihood estimation and expectation-maximization algorithm} \label{review-em}

To make this paper self-contained, we review MLE and EM.
Suppose that we have $N$ data points $Y_\mathrm{obs} = \{y_{1}, y_{2}, \dots, y_{N}\}$, which are independent and identically distributed, and let $\{\sigma_1, \sigma_2, \dots, \sigma_N\}$ be a set of hidden variables.
In this paper, we denote the joint probability density function on $y_i$ and $\sigma_i$ with a parameter $\theta$ as $f(y_i, \sigma_i; \theta)$.

Using these definitions, the log-likelihood function of $Y_\mathrm{obs}$ is represented by
\begin{align}
\mathcal{K} (\theta) &\coloneqq \sum_{i=1}^N \ln \sum_{\sigma_i \in S^\sigma} f (y_i, \sigma_i; \theta) , \label{log-likelihood-xx-01}
\end{align}
where $S^\sigma$ represents a discrete configuration set of $\sigma_i$; that is, $S^\sigma = \{1, 2, \dots, K\}$.
In MLE, we estimate the parameter $\theta$ by maximizing the log-likelihood function, Eq.~\eqref{log-likelihood-xx-01}.
However, it is difficult to maximize Eq.~\eqref{log-likelihood-xx-01}, because $\mathcal{K}(\theta)$ is analytically unsolvable and primitive methods, such as Newton's method, are known to be less effective~\cite{Xu01}.
We therefore often use EM for practical applications~\cite{Bishop01, Murphy01}.

EM consists of two steps.
To introduce them, we decompose $\mathcal{K}(\theta)$ into two parts:
\begin{align}
\mathcal{K} (\theta) &= \mathcal{L}(\theta) + \mathrm{KL} \left( \prod_{i=1}^N q(\sigma_{i}) \middle\| \prod_{i=1}^N \frac{f(y_i, \sigma_i; \theta)}{\sum_{\sigma_i \in S^\sigma} f(y_i, \sigma_i; \theta)} \right),
\end{align}
where
\begin{align}
\mathcal{L}(\theta) &\coloneqq \sum_{i=1}^N \sum_{\sigma_{i} \in S^\sigma} q(\sigma_{i}) \ln \left( f(y_{i}, \sigma_{i}; \theta) \frac{1}{q(\sigma_{i})} \right), \label{EM-Estep-02} \\
\mathrm{KL} &\left( \prod_{i=1}^N q(\sigma_{i}) \middle\| \prod_{i=1}^N \frac{f(y_i, \sigma_i; \theta)}{\sum_{\sigma_i \in S^\sigma} f(y_i, \sigma_i; \theta)} \right) \nonumber \\
            &\coloneqq - \sum_{i=1}^N \sum_{\sigma_{i} \in S^\sigma} q(\sigma_{i}) \ln \left( \frac{f(y_i, \sigma_i; \theta)}{\sum_{\sigma_i \in S^\sigma} f(y_i, \sigma_i; \theta)} \frac{1}{q(\sigma_{i})} \right). \label{KL-xx-01}
\end{align}
Here, we use an arbitrary probability function $q(\sigma_i)$ that satisfies $\sum_{\sigma_i \in S^\sigma} q(\sigma_i) = 1$.
Note that $\mathrm{KL}(\cdot \| \cdot)$ is the Kullback-Leibler divergence~\cite{Kullback01, Kullback02}.
On the basis of this decomposition, EM is composed of the following two steps, which are called the E and M steps.
In the E step, we minimize the KL divergence, Eq~\eqref{KL-xx-01}, with respect to $q(\sigma_i)$ under a fixed $\theta'$.
Then we obtain
\begin{align}
q(\sigma_i) &= \frac{f(y_i, \sigma_i; \theta')}{\sum_{\sigma_i \in S^\sigma} f(y_i, \sigma_i; \theta')}. \label{EM-Estep01}
\end{align}
Next, by substituting Eq.~\eqref{EM-Estep01} into Eq.~\eqref{EM-Estep-02}, we obtain
\begin{align}
\mathcal{L}(\theta) = \mathcal{Q}(\theta, \theta') - \sum_{i=1}^N &\sum_{\sigma_i \in S^\sigma} \frac{f(y_i, \sigma_i; \theta')}{\sum_{\sigma_i \in S^\sigma} f(y_i, \sigma_i; \theta')} \ln \frac{f(y_i, \sigma_i; \theta')}{\sum_{\sigma_i \in S^\sigma} f(y_i, \sigma_i; \theta')},
\end{align}
where
\begin{align}
\mathcal{Q}(\theta, \theta') &\coloneqq \sum_{i=1}^N \sum_{\sigma_i \in S^\sigma} \frac{f(y_i, \sigma_i; \theta')}{\sum_{\sigma_i \in S^\sigma} f(y_i, \sigma_i; \theta')} \ln f(y_i, \sigma_i; \theta). \label{EM-Mstep01}
\end{align}
In the M step, we maximize $\mathcal{L}(\theta)$ instead of $\mathcal{K}(\theta)$ with respect to $\theta$; that is, we maximize $\mathcal{Q}(\theta, \theta')$ under the fixed $\theta'$.
In EM, we iterate these two steps.
Thus, assuming that $\theta_{t}$ be a tentative estimated parameter at the $t$-th iteration, the new estimated parameter $\theta_{t+1}$ is determined by
\begin{align}
\theta_{t+1} = \argmax_\theta \, \mathcal{Q}(\theta; \theta_{t}). \label{EM-Mstep03}
\end{align}
This procedure is summarized in Algo.~\ref{EM-algorithm-01} with pseudo-code.
\begin{algorithm}[t]
\caption{Expectation-maximization (EM) algorithm}
\label{EM-algorithm-01}
\begin{algorithmic}[1]
\STATE set $t \leftarrow 0$ and initialize $\theta_{0}$
\WHILE{convergence criterion is satisfied}
\STATE (E step) calculate $q(\sigma_{i})$ in Eq.~\eqref{EM-Estep01} with $\theta = \theta_t$ for $i = 1, 2, \dots, N$
\STATE (M step) calculate $\theta_{t+1}$ with Eq.~\eqref{EM-Mstep03}
\ENDWHILE
\end{algorithmic}
\end{algorithm}
Despite the success of EM, it is known that EM sometimes trapped in local optima and fails to estimate the optimal parameter~\cite{Bishop01, Murphy01}.
To relax these problems, we improve EM by employing methods in quantum mechanics.

\section{Deterministic quantum annealing expectation-maximization algorithm} \label{dqaem-gmm-00}

This section is the main part of this paper.
We formulate DQAEM by quantizing the hidden variables $\{\sigma_i\}_{i = 1}^N$ in EM and employing the annealing technique.
In the previous section, we denoted by $\sigma_i$ and $S^\sigma$ the hidden variable for each data point $i$ and the set of its possible values, respectively.
Corresponding to the classical setup, we introduce a quantum one as follows.
We define an operator $\hat{\sigma}_i$ and a ket vector $\Ket{\sigma_i = k}$ $(k \in S^\sigma)$ so that they satisfy
\begin{align}
\hat{\sigma}_i \ket{\sigma_i = k} &= k \ket{\sigma_i = k}.
\end{align}
Here we note that the eigenvalues of $\hat{\sigma}_i$ correspond to the possible values of the hidden variable $\sigma_i$ in the classical setup.
In addition, we introduce a bra vector corresponding to $\Ket{\sigma_i = k}$ as $\Bra{\sigma_i = l}$ $(l \in S^\sigma)$ so that it satisfies
\begin{align}
\Braket{\sigma_i = k | \sigma_i = l}
&=
\begin{cases}
1 \ (k = l) \\
0 \ (k \ne l)
\end{cases}.
\end{align}
Moreover, we denote the trace on $\hat{\sigma}_i$ by $\mathrm{Tr}_{\sigma_i} [\cdot]$; using the ket vector $\Ket{\sigma_i = k}$ and the bra vector $\Bra{ \sigma_i = k}$, it is represented by $\mathrm{Tr}_{\sigma_i} [\cdot] = \sum_{k = 1}^K \Braket{ \sigma_i = k | \cdot | \sigma_i = k}$.
To simplify the notation, we sometimes use $\Ket{\sigma_i}$ and $\Bra{ \sigma_i}$ for $\Ket{\sigma_i = k}$ and $\Bra{ \sigma_i = k}$, respectively, and the trace on $\sigma_i$ is also represented by $\mathrm{Tr}_{\sigma_i} [\cdot] = \sum_{\sigma_i \in S^\sigma} \Braket{ \sigma_i | \cdot | \sigma_i}$.

Next, we define the negative free energy function as
\begin{align}
\mathcal{G}_{\beta, \Gamma} (\theta) &\coloneqq \frac{1}{\beta} \ln \mathcal{Z}_{\beta, \Gamma} (\theta), \label{free-energy-66}
\end{align}
where $\mathcal{Z}_{\beta, \Gamma} (\theta)$ denotes the partition function that is given by
\begin{align}
\mathcal{Z}_{\beta, \Gamma} (\theta) &\coloneqq \prod_{i=1}^N \mathcal{Z}_{\beta, \Gamma}^{(i)} (\theta), \\
\mathcal{Z}_{\beta, \Gamma}^{(i)} (\theta) &\coloneqq \mathrm{Tr}_{\sigma_i} \left[ f_{\beta, \Gamma} (y_{i}, \hat{\sigma}_{i}; \theta) \right].
\end{align}
Here, $f_{\beta, \Gamma} (y_{i}, \hat{\sigma}_{i}; \theta)$ is the exponential weight with a non-commutative term $H^\mathrm{nc}$:
\begin{align}
		f_{\beta, \Gamma} (y_{i}, \hat{\sigma}_{i}; \theta) &\coloneqq \exp \left( - \beta \{ H + H^\mathrm{nc} \} \right), \\
H(y_{i}, \hat{\sigma}_{i}; \theta) &\coloneqq - \ln f(y_{i}, \hat{\sigma}_{i}; \theta), \label{hamil-def-01} \\
H^\mathrm{nc} &\coloneqq \Gamma \hat{\sigma}^\mathrm{nc},
\end{align}
where the non-commutative relation $[\hat{\sigma}_i, \hat{\sigma}^\mathrm{nc}] \ne 0$ is imposed, and $\beta$ and $\Gamma$ represent parameters for simulated and quantum annealing, respectivlely.
If we set $\Gamma = 0$ and $\beta = 1$, the negative free energy function, Eq.~\eqref{free-energy-66}, reduces to the log-likelihood function, Eq.~\eqref{log-likelihood-xx-01}:
\begin{align}
\mathcal{G}_{\beta = 1, \Gamma = 0} (\theta) = \mathcal{K} (\theta).
\end{align}
Therefore, in annealing processes, $\Gamma$ and $\beta$ are changed from appropriate initial values to $0$ and $1$, respectively.
Corresponding to EM in the previous section, we construct the E and M steps in DQAEM.
Using an arbitrary density matrix $\hat{\rho}_i$ that satisfies $\mathrm{Tr}_{\sigma_i} [\hat{\rho}_i] = 1$, we divide the negative free energy function, Eq.~\eqref{free-energy-66}, into two parts as
\begin{align}
\beta \mathcal{G}_{\beta,\Gamma}(\theta) &= \mathcal{F}_{\beta,\Gamma}(\theta) + \mathrm{KL}\left( \hat{\rho}_i \middle\| \frac{f_{\beta, \Gamma} (y_{i}, \hat{\sigma}_{i}; \theta)}{\mathrm{Tr}_{\sigma_i} \left[ f_{\beta, \Gamma} (y_{i}, \hat{\sigma}_{i}; \theta) \right] } \right),
\end{align}
where
\begin{align}
\mathcal{F}_{\beta, \Gamma} (\theta) &\coloneqq \sum_{i=1}^N \mathrm{Tr}_{\sigma_{i}} \left[ \hat{\rho}_{i} \left\{ \ln f_{\beta, \Gamma} (y_{i}, \hat{\sigma}_{i}; \theta) - \ln \hat{\rho}_i \right\} \right], \label{QAEM-Mstep31}
\end{align}
\begin{align}
\mathrm{KL} &\left( \hat{\rho}_i \middle\| \frac{f_{\beta, \Gamma} (y_{i}, \hat{\sigma}_{i}; \theta)}{\mathrm{Tr}_{\sigma_i} \left[ f_{\beta, \Gamma} (y_{i}, \hat{\sigma}_{i}; \theta) \right] } \right) \nonumber \\
            &\coloneqq - \sum_{i=1}^N \mathrm{Tr}_{\sigma_{i}} \bigg[ \hat{\rho}_{i} \bigg\{ \ln \frac{f_{\beta, \Gamma} (y_{i}, \hat{\sigma}_{i}; \theta)}{\mathrm{Tr}_{\sigma_i} \left[ f_{\beta, \Gamma} (y_{i}, \hat{\sigma}_{i}; \theta) \right] } - \ln \hat{\rho}_i \bigg\} \bigg]. \label{quantum-KL-01}
\end{align}
Here, $\mathrm{KL} (\hat{\rho}_1 \| \hat{\rho}_2) = \mathrm{Tr}[\hat{\rho}_1 (\ln \hat{\rho}_1 - \ln \hat{\rho}_2)]$ is a quantum extension of the Kullback-Leibler divergence between density matrices $\hat{\rho}_1$ and $\hat{\rho}_2$~\cite{Umegaki01}.

In the E step of DQAEM, Eq.~\eqref{quantum-KL-01} under a fixed $\theta$ is minimized; then we obtain
\begin{align}
\hat{\rho}_i &= \frac{f_{\beta, \Gamma} (y_{i}, \hat{\sigma}_{i}; \theta)}{\mathrm{Tr}_{\sigma_i} \left[ f_{\beta, \Gamma} (y_{i}, \hat{\sigma}_{i}; \theta) \right] }, \label{posterior02}
\end{align}
which is a quantum extension of Eq.~\eqref{EM-Estep01}.
Note that, to calculate $f_{\beta, \Gamma} (y_{i}, \hat{\sigma}_{i}; \theta)$, we need to use the Suzuki-Trotter expansion~\cite{Suzuki01}, the Taylor expansion, the Pad\'e approximation or the technique proposed in Ref.~\cite{Al-Mohy01}, because $f_{\beta, \Gamma} (y_{i}, \hat{\sigma}_{i}; \theta')$ has the non-commutative term $H^\mathrm{nc}$.

On the other hand, in the M step of DQAEM, the parameter $\theta$ is determined by maximizing Eq.~\eqref{QAEM-Mstep31}.
Following the method in the previous section, we substitute Eq.~\eqref{posterior02} into Eq.~\eqref{QAEM-Mstep31}; then we have
\begin{align}
\ifthenelse{\versionmath = 1}
{
\mathcal{F}_{\beta,\Gamma}(\theta) &= \beta \mathcal{U}_{\beta,\Gamma}(\theta; \theta') \nonumber \\
& \quad - \sum_{i=1}^N \mathrm{Tr}_{\sigma_i} \Bigg[ \frac{f_{\beta, \Gamma}(y_i, \sigma_i; \theta')}{\mathrm{Tr}_{\sigma_i} f_{\beta, \Gamma}(y_i, \sigma_i; \theta')} \nonumber \\
& \qquad \qquad \times \ln \frac{f_{\beta, \Gamma}(y_i, \sigma_i; \theta')}{\mathrm{Tr}_{\sigma_i} f_{\beta, \Gamma}(y_i, \sigma_i; \theta')} \Bigg],
}
{
\mathcal{F}_{\beta,\Gamma}(\theta) &= \mathcal{U}_{\beta,\Gamma}(\theta; \theta') - \sum_{i=1}^N \mathrm{Tr}_{\sigma_i} \Bigg[ \frac{f_{\beta, \Gamma}(y_i, \sigma_i; \theta')}{\mathrm{Tr}_{\sigma_i} f_{\beta, \Gamma}(y_i, \sigma_i; \theta')} \ln \frac{f_{\beta, \Gamma}(y_i, \sigma_i; \theta')}{\mathrm{Tr}_{\sigma_i} f_{\beta, \Gamma}(y_i, \sigma_i; \theta')} \Bigg],
}
\end{align}
where
\begin{align}
\ifthenelse{\versionmath = 1}
{
&\mathcal{U}_{\beta,\Gamma} (\theta; \theta') \coloneqq \sum_{i=1}^N \mathrm{Tr}_{\sigma_{i}} \left[ \frac{f_{\beta, \Gamma} (y_{i}, \hat{\sigma}_{i}; \theta')}{\mathrm{Tr}_{\sigma_i} \left[ f_{\beta, \Gamma} (y_{i}, \hat{\sigma}_{i}; \theta') \right] } \ln f_{\beta = 1, \Gamma} (y_{i}, \hat{\sigma}_{i}; \theta) \right]. \label{QAEM-Mstep01}
}
{
&\mathcal{U}_{\beta,\Gamma} (\theta; \theta') \coloneqq \sum_{i=1}^N \mathrm{Tr}_{\sigma_{i}} \left[ \frac{f_{\beta, \Gamma} (y_{i}, \hat{\sigma}_{i}; \theta')}{\mathrm{Tr}_{\sigma_i} \left[ f_{\beta, \Gamma} (y_{i}, \hat{\sigma}_{i}; \theta') \right] } \ln f_{\beta, \Gamma} (y_{i}, \hat{\sigma}_{i}; \theta) \right]. \label{QAEM-Mstep01}
}
\end{align}
Note that $\mathcal{U}_{\beta,\Gamma} (\theta; \theta')$ represents the quantum extension of $\mathcal{Q}(\theta, \theta')$, Eq.~\eqref{EM-Mstep01}.
Thus, the computation of the M step of DQAEM is written as
\begin{align}
\theta_{t+1} = \argmax_\theta \, \mathcal{U}_{\beta, \Gamma} (\theta, \theta_{t}). \label{QAEM-Mstep03}
\end{align}
During iterations, the annealing parameter $\Gamma$, which controls the strength of quantum fluctuations, is changed from the initial value $\Gamma_0 \ (\Gamma_0 \ge 0)$ to $0$, and the other parameter $\beta$, which controls the strength of thermal fluctuations, is changed from the initial value $\beta_0 \ (0< \beta_0 \le 1)$ to $1$.
We summarize DQAEM in Algo.~\ref{DQAEM-algorithm-01}.
\begin{algorithm}[t]
\caption{Deterministic quantum annealing expectation-maximization (DQAEM) algorithm}
\label{DQAEM-algorithm-01}
\begin{algorithmic}[1]
\STATE set $t \leftarrow 0$ and initialize $\theta_{0}$
\STATE set $\beta \leftarrow \beta_0 \ (0 < \beta_0 \le 1)$ and $\Gamma \leftarrow \Gamma_0 \ (\Gamma_0 \ge 0)$
\WHILE{convergence criteria is satisfied}
\STATE (E step) calculate $\hat{\rho}_{i}$ in Eq.~\eqref{posterior02} with $\theta = \theta_t$ for $i = 1, 2, \dots, N$
\STATE (M step) calculate $\theta_{t+1}$ with Eq.~\eqref{QAEM-Mstep03}
\STATE increase $\beta$ and decrease $\Gamma$
\ENDWHILE
\end{algorithmic}
\end{algorithm}

Finally, we mention that the free energy function, Eq.~\eqref{free-energy-66}, increases monotonically when the parameter $\theta_t$ is updated by DQAEM.
The proof is given in~\ref{app-dqaem-gmm-02}.

\section{Mechanism of DQAEM} \label{path-dqaem-10}

In this section, we explain an advantage of DQAEM over EM and DSAEM using the path integral formulation.
First, we demonstrate that DQAEM is a quantum extension of EM and DSAEM.
Setting $\Gamma = 0$ in Eq.~\eqref{QAEM-Mstep01}, which corresponds to the classical limit, we have
\begin{align}
\ifthenelse{\versionmath = 1}
{
\mathcal{U}_{\beta,\Gamma = 0} (\theta; \theta') &= \sum_{i=1}^N \frac{1}{\sum_{\sigma_i \in S^\sigma} \Braket{\sigma_i | f_{\beta, \Gamma = 0} (y_{i}, \hat{\sigma}_{i}; \theta') | \sigma_i}} \nonumber \\
& \qquad \times \sum_{\sigma_{i} \in S^\sigma} \Braket{ \sigma_i | f_{\beta, \Gamma = 0} (y_{i}, \hat{\sigma}_{i}; \theta') \ln f_{\beta = 1, \Gamma = 0} (y_{i}, \hat{\sigma}_{i}; \theta)  | \sigma_i },
}
{
\mathcal{U}_{\beta,\Gamma = 0} (\theta; \theta') &= \sum_{i=1}^N \frac{1}{\sum_{\sigma_i \in S^\sigma} \Braket{ \sigma_i | f_{\beta, \Gamma = 0} (y_{i}, \hat{\sigma}_{i}; \theta') | \sigma_i }} \nonumber \\
& \qquad \times \sum_{\sigma_{i} \in S^\sigma} \Braket{ \sigma_i | f_{\beta, \Gamma = 0} (y_{i}, \hat{\sigma}_{i}; \theta') \ln f_{\beta, \Gamma = 0} (y_{i}, \hat{\sigma}_{i}; \theta)  | \sigma_i },
}
\end{align}
where we use $\mathrm{Tr}_{\sigma_i} [\cdot] = \sum_{\sigma_i \in S^\sigma} \Braket{ \sigma_i | \cdot | \sigma_i }$.
Taking into account that $f_{\beta, \Gamma = 0} (y_{i}, \hat{\sigma}_{i}; \theta')$ does not have the non-commutative term $H^\mathrm{nc}$, we get
\begin{align}
\ifthenelse{\versionmath = 1}
{
\mathcal{U}_{\beta,\Gamma = 0} (\theta; \theta') &= \sum_{i=1}^N \frac{\sum_{\sigma_{i} \in S^\sigma} f_{\beta, \Gamma = 0} (y_{i}, \sigma_{i}; \theta')  \ln f_{\beta = 1, \Gamma = 0} (y_{i}, \sigma_{i}; \theta)}{\sum_{\sigma_i \in S^\sigma} f_{\beta, \Gamma = 0} (y_{i}, \sigma_{i}; \theta') }. \label{new-01}
}
{
\mathcal{U}_{\beta,\Gamma = 0} (\theta; \theta') &= \sum_{i=1}^N \frac{\sum_{\sigma_{i} \in S^\sigma} f_{\beta, \Gamma = 0} (y_{i}, \sigma_{i}; \theta')  \ln f_{\beta, \Gamma = 0} (y_{i}, \sigma_{i}; \theta)}{\sum_{\sigma_i \in S^\sigma} f_{\beta, \Gamma = 0} (y_{i}, \sigma_{i}; \theta') }. \label{new-01}
}
\end{align}
Equation~\eqref{QAEM-Mstep03} with Eq.~\eqref{new-01} provides the update equation for $\theta_{t+1}$ in DSAEM, and we obtain $\theta_{t+1}$ by solving
\begin{align}
\sum_{i=1}^N \sum_{\sigma_i \in S^\sigma} \bigg\{ \frac{f_{\beta, \Gamma = 0} (y_i, {\sigma}_i; \theta_t)}{\sum_{\sigma_i \in S^\sigma} f_{\beta, \Gamma = 0} (y_i, \sigma_i; \theta_t) } \frac{d}{d \theta} H (y_i, {\sigma}_i; \theta) \bigg\} &= 0. \label{new-02}
\end{align}
This update rule is equivalent to that of DSAEM~\cite{Ueda02}.
Furthermore, if we set $\beta = 1$, Eq.~\eqref{new-01} equals to Eq.~\eqref{EM-Mstep01} and DSAEM also reduces to EM.

On the other hand, in DQAEM, the parameter $\theta$ is updated based on Eq.~\eqref{QAEM-Mstep03} with Eq.~\eqref{QAEM-Mstep01}.
Thus we obtain $\theta_{t+1}$ by solving
\begin{align}
\sum_{i=1}^N \mathrm{Tr}_{\sigma_i} \bigg[ \frac{f_{\beta, \Gamma} (y_i, \hat{\sigma}_i; \theta_t)}{\mathrm{Tr}_{\sigma_i} \left[ f_{\beta, \Gamma} (y_i, \hat{\sigma}_i; \theta_t) \right]} \frac{d}{d \theta} H (y_i, \hat{\sigma}_i; \theta) \bigg] &= 0. \label{update-DQAEM-xx-01}
\end{align}
Using the bra-ket notation, Eq.~\eqref{update-DQAEM-xx-01} can be arranged as
\begin{align}
\sum_{i=1}^N \sum_{\sigma_i \in S^\sigma} & \bigg\{ \frac{ \Braket{ \sigma_i | f_{\beta, \Gamma} (y_i, \hat{\sigma}_i; \theta')| \sigma_i}}{\sum_{\sigma_i \in S^\sigma} \Braket{ \sigma_i | f_{\beta, \Gamma} (y_i, \hat{\sigma}_i; \theta') | \sigma_i }} \frac{d}{d \theta} H (y_i, \sigma_i; \theta) \bigg\} = 0. \label{update-DQAEM-xx-02}
\end{align}

The difference between Eqs.~\eqref{new-02} and \eqref{update-DQAEM-xx-02} is the existence of the non-commutative term $H^\mathrm{nc}$ in $f_{\beta, \Gamma} (y_i, \hat{\sigma}_i; \theta')$, which gives rise to the advantage of DQAEM.
To evaluate the effect of $H^\mathrm{nc}$,
we calculate $\Braket{ \sigma_i | f_{\beta, \Gamma} (y_i, \hat{\sigma}_i; \theta') | \sigma_i }$ by applying the Suzuki-Trotter expansion~\cite{Suzuki01}.
Then we obtain
\begin{align}
&\Braket{ \sigma_i | f_{\beta, \Gamma} (y_i, \hat{\sigma}_i; \theta') | \sigma_i } \nonumber \\
& \quad = \lim_{M \rightarrow \infty} \sum_{\sigma_{i, 1}', \sigma_{i, 1}, \dots, \sigma_{i, M-1}, \sigma_{i, M}' \in S^\sigma} \prod_{j=1}^M \Braket{ \sigma_{i, j} | e^{ - \frac{\beta}{M} H(y_i, \hat{\sigma}_i; \theta')} | \sigma_{i, j}'} \Braket{ \sigma_{i, j}' | e^{ - \frac{\beta}{M} H^\mathrm{nc}} | \sigma_{i, j-1}}, \\
& \quad = \lim_{M \rightarrow \infty} \sum_{\sigma_{i, 1}, \sigma_{i, 2}, \dots, \sigma_{i, M-1} \in S^\sigma} \prod_{j=1}^M \Braket{ \sigma_{i, j} | e^{ - \frac{\beta}{M} H(y_i, \hat{\sigma}_i; \theta')} | \sigma_{i, j}} \Braket{ \sigma_{i, j} | e^{ - \frac{\beta}{M} H^\mathrm{nc}} | \sigma_{i, j-1}}, \label{path-integral-06-01}
\end{align}
with the boundary conditions $\Ket{ \sigma_{i, 0} } = \Ket{ \sigma_{i, M} } = \Ket{ \sigma_i }$ and $\Bra{ \sigma_{i, 0} } = \Bra{ \sigma_{i, M} } = \Bra{ \sigma_i }$.
In Eq.~\eqref{path-integral-06-01}, the quantum effect comes from $\Braket{ \sigma_{i, j} | e^{ - \frac{\beta}{M} H^\mathrm{nc}} | \sigma_{i, j-1} }$.
If we assume the classical case $H^\mathrm{nc} = 0$ (i.e. $\Gamma = 0$), we obtain $\Braket{ \sigma_{i, j} | e^{ - \frac{\beta}{M} H^\mathrm{nc}} | \sigma_{i, j-1} } = \Braket{ \sigma_{i, j} | \sigma_{i, j-1} } = \delta_{\sigma_{i,j}, \sigma_{i,j-1}}$, where $\delta_{\cdot, \cdot}$ is the Kronecker delta function.
Thus, $\Ket{ \sigma_{i, j} }$ does not depend on the index along the Trotter dimension, $j$.
In terms of the path integral formulation, this fact implies that $\Braket{ \sigma_i | f_{\beta, \Gamma} (y_i, \hat{\sigma}_i; \theta') | \sigma_i}$ can be evaluated by a single classical path with the boundary condition fixed at $\sigma_i$; see Fig.~\ref{numerical-15-01}.
On the other hand, in the quantum case, $\Braket{ \sigma_{i, j} | e^{ - \frac{\beta}{M} H^\mathrm{nc}} | \sigma_{i, j-1}} \ne \Braket{ \sigma_{i, j} | \sigma_{i, j-1} }$, and therefore Eq.~\eqref{path-integral-06-01} involves not only the classical path but also quantum paths, which depend on the form of $H^\mathrm{nc}$; see Fig.~\ref{numerical-15-01}.
%
Thus, owing to these quantum paths, DQAEM may overcome the problem of local optima and it is expected that DQAEM outperforms EM and DSAEM. 
In Sec.~\ref{sec-numerical-00}, we show that the quantum effect really helps EM to find the global optimum through numerical simulations.

Before closing this section, we represent Eq.~\eqref{path-integral-06-01} with the path integral form.
In the limit $M \rightarrow \infty$, we define
\begin{align}
\Ket{ \sigma_{i}(\tau) } &\coloneqq \Ket{ \sigma_{i, j} },
\end{align}
where $\tau = \beta j / M$.
This notation leads to another expression of Eq.~\eqref{path-integral-06-01} as
\begin{align}
\Braket{ \sigma_i | f_{\beta, \Gamma} (y_i, \hat{\sigma}_i; \theta') | \sigma_i } &= \quad \int_{\sigma_i(0) = \sigma_i(\beta) = \sigma_i} \mathcal{D} \sigma_i(\tau) \, e^{ - S[\sigma_i(\tau)] },
\end{align}
where $\int \mathcal{D} \sigma_i(\tau)$ represents the integration over all paths and $S[\sigma_i(\tau)]$ is the action given by
\begin{align}
S[\sigma_i(\tau)] = - \int_{0}^{\beta} d\tau \, \Braket{ \sigma_{i}(\tau) | \left[ \frac{d}{d \tau} - \left\{ H(y_i, \hat{\sigma}_i; \theta') + H^\mathrm{nc} \right\} \right] | \sigma_{i}(\tau) }.
\end{align}
Note that this action also appears in the discussion of the Berry phase~\cite{Berry01, Nagaosa01}.

\begin{figure}[t]
\centering
\includegraphics[scale=1.00]{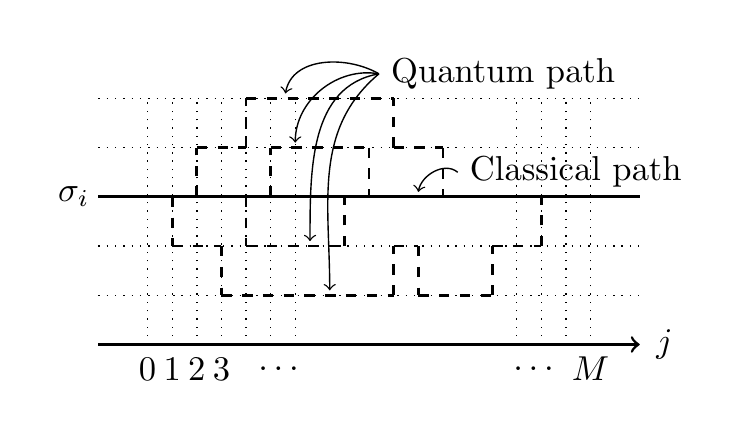}
\caption{
Schematic of classical paths (solid line) and quantum paths (dashed line) with the boundary condition $\sigma_i = \sigma_{i, 0} = \sigma_{i, M}$.
The axis represents the Trotter dimension $j \ (0 \le j \le M)$.
Quantum paths have jumps from state to state while classical paths do not.
In EM and DSAEM, only classical path (solid line) contribute to Eq.~\eqref{path-integral-06-01}.
In contrast, in DQAEM, we need to sum over all quantum paths (dashed line).
}
\label{numerical-15-01}
\end{figure}

\section{Numerical simulations} \label{sec-numerical-00}

In this section, we present numerical simulations to confirm the performance of DQAEM.
Here, we deal with GMMs with hidden variables.
We introduce the quantum representation of GMMs in Sec.~\ref{gmm-01}, and give update equations in Sec.~\ref{sec-numerical-01}.
In Sec.~\ref{sec-numerical-02}, we show numerical results to elucidate the advantage of DQAEM over EM and DSAEM.

\subsection{Gaussian mixture models} \label{gmm-01}

First, we provide a concrete setup of GMMs.
Consider that a GMM is composed of $K$ Gaussian functions and the domain of the hidden variables $S^\sigma$ is given by $\{1, 2, \dots, K\}$~\footnote{When we work on the one-hot notation~\cite{Bishop01, Murphy01}, we can formulate an equivalent quantization scheme.~\cite{Miyahara04}}.
Then the distribution of the GMM is written by
\begin{align}
f(y_i, \sigma_i; \theta) &= \sum_{k = 1}^K \pi^{k} g(y_i; \mu^{k}, \Sigma^k) \delta_{\sigma_i, k}, \label{joint-xx-01}
\end{align}
where	$g^k(y_i; \mu^{k}, \Sigma^k)$ is the $k$-th Gaussian function of the GMM parametrized by $\{\mu^k, \Sigma^k\}$, and $\delta_{\cdot, \cdot}$ is the Kronecker delta.
The parameters $\{\pi_k\}_{k=1}^K$ stand for the mixing ratios that satisfy $\sum_{k=1}^K \pi^k = 1$.
Here, we denote all the parameters collectively by $\theta$: $\theta = \{\pi^k, \mu^k, \Sigma^k\}_{k=1}^K$.
From Eq.~\eqref{hamil-def-01}, we get the Hamiltonian of this system:
\begin{align}
H(y_i, \sigma_i; \theta) &= \sum_{k=1}^K h_i^{k} \delta_{\sigma_i, k}, \label{hamil-classical-xx-01}
\end{align}
where we take the logarithm of Eq.~\eqref{joint-xx-01} and $h_i^k = - \ln ( \pi^{k} g(y_i; \mu^k, \Sigma^k))$.

Next, following the discussion of Sec.~\ref{dqaem-gmm-00}, we define the ket vector $\Ket{ \sigma_i }$, and the operator $\hat{\sigma}_i$ as
\begin{align}
\hat{\sigma}_i \Ket{ \sigma_i = k } &= k \Ket{ \sigma_i = k },
\end{align}
where $k$ represents the label of components of the GMM; that is, $k \in S^\sigma$.
To quantize Eq.~\eqref{hamil-classical-xx-01}, we replace $\sigma_i$ by $\hat{\sigma}_i$.
%
Then, we obtain the quantum Hamiltonian of the GMM as
\begin{align}
H(y_i, \hat{\sigma}_i; \theta) &= \sum_{k=1}^K h_i^{k} \delta_{\hat{\sigma}_i, k} \label{hamil-quantum-48} \\
 															&= \sum_{k=1}^K h_i^{k} \Ket{\sigma_i = k } \Bra{ \sigma_i = k }, \label{hamil-quantum-55} 
\end{align}
where we use $\delta_{\hat{\sigma}_i, k} = \Ket{ \sigma_i = k } \Bra{ \sigma_i = k }$.
If we adopt the representation,
\begin{align}
\Ket{ \sigma_i = k } &= [\underbrace{0, \dots, 0}_{k - 1}, 1, \underbrace{0, \dots, 0}_{K - k}]^\intercal,
\end{align}
Eq.~\eqref{hamil-quantum-55} can be expressed as
\begin{align}
H(y_i, \hat{\sigma}_i; \theta) &= \textrm{diag}(h_i^1, h_i^2, \dots, h_i^{K}).
\end{align}

Finally we add a non-commutative term $H^\mathrm{nc} = \Gamma \hat{\sigma}^\mathrm{nc}$ with $[\hat{\sigma}_i, \hat{\sigma}^\mathrm{nc}] \ne 0$ to Eq.~\eqref{hamil-quantum-48}.
%
Obviously, there are many candidates for $\hat{\sigma}^\mathrm{nc}$; however, in the numerical simulation, we adopt
\begin{align}
\hat{\sigma}^\mathrm{nc} &= \sum_{\substack{k = 1, \dots, K \\ l = k \pm 1}} \Ket{ \sigma_i = l } \Bra{ \sigma_i = k },
\end{align}
where $\Ket{ \sigma_i = 0 } = \Ket{ \sigma_i = K }$ and $\Ket{ \sigma_i = K + 1 } = \Ket{ \sigma_i = 1 }$.
This term induces quantum paths shown in Fig.~\ref{numerical-15-01}.

\subsection{Update equations} \label{sec-numerical-01}

Suppose that we have $N$ data points $Y_\mathrm{obs} = \{y_1, y_2, \dots, y_N\}$, and they are independent and identically distributed obeying a GMM.
In EM, the update equations of GMMs are obtained by Eq.~\eqref{EM-Mstep03} with the constraint $\sum_{k=1}^K \pi^k = 1$; then the parameters at the $(t+1)$-th iteration, $\theta_{t+1}$, are given by
\begin{align}
\pi_{t+1}^k &= \frac{N_k}{N}, \label{update-EM-01} \\
\mu_{t+1}^k &= \frac{1}{N_k} \sum_{i=1}^N y_{i} \frac{f(y_{i}, \sigma_{i} = k; \theta_{t})}{\sum_{\sigma_i \in S^\sigma} f(y_{i}, \sigma_{i}; \theta_{t})}, \label{update-EM-02} \\
\Sigma_{t+1}^k &= \frac{1}{N_k} \sum_{i=1}^N (y_{i} - \mu_{t+1}^k) (y_{i} - \mu_{t+1}^k)^\intercal \frac{f(y_{i}, \sigma_{i} = k; \theta_{t})}{\sum_{\sigma_i \in S^\sigma} f(y_{i}, \sigma_{i}; \theta_{t})}, \label{update-EM-03}
\end{align}
where $N_k = \sum_{i=1}^N \frac{f(y_{i}, \sigma_{i} = k; \theta_{t})}{\sum_{\sigma_i \in S^\sigma} f(y_{i}, \sigma_{i}; \theta_{t})}$, $\theta_{t}$ is the tentative estimated parameter at the $t$-th iteration.

In DQAEM, from Eq.~\eqref{QAEM-Mstep03}, the update equations are expressed by
\begin{align}
\pi_{t+1}^k &= \frac{N_k^\mathrm{QA}}{N} , \label{update-QA-01} \\
\mu_{t+1}^k &= \frac{1}{N_k^\mathrm{QA}} \sum_{i=1}^N y_{i} \frac{f_{\beta, \Gamma}^\mathrm{QA} (y_{i}, \sigma_{i} = k; \theta_{t})}{\mathrm{Tr}_{\sigma_i} \left[ f_{\beta, \Gamma} (y_{i}, \hat{\sigma}_{i}; \theta_{t}) \right]}, \label{update-QA-02} \\
\Sigma_{t+1}^k &= \frac{1}{N_k^\mathrm{QA}} \sum_{i=1}^N (y_{i} - \mu_{t+1}^k) (y_{i} - \mu_{t+1}^k)^\intercal \frac{f_{\beta, \Gamma}^\mathrm{QA} (y_{i}, \sigma_{i} = k; \theta_{t})}{\mathrm{Tr}_{\sigma_i} \left[ f_{\beta, \Gamma} (y_{i}, \hat{\sigma}_{i}; \theta_{t}) \right]}, \label{update-QA-03}
\end{align}
where $N_k^\mathrm{QA} = \sum_{i=1}^N \frac{ f_{\beta, \Gamma}^\mathrm{QA} (y_{i}, \sigma_{i} = k; \theta_{t}) }{\mathrm{Tr}_{\sigma_i} \left[ f_{\beta, \Gamma} (y_{i}, \hat{\sigma}_{i}; \theta_{t}) \right] }$ and we denote $\Braket{ \sigma_i = k | f_{\beta, \Gamma}(y_{i}, \hat{\sigma}_{i}; \theta_{t}) | \sigma_i = k }$ by $f_{\beta, \Gamma}^\mathrm{QA} (y_{i}, \sigma_{i} = k; \theta_{t})$ for simplicity.
If we set $\Gamma = 0$ and $\beta = 1$, the update equations of DQAEM, Eqs.~\eqref{update-QA-01}, \eqref{update-QA-02}, and \eqref{update-QA-03} reduce to those  of EM, Eqs.~\eqref{update-EM-01}, \eqref{update-EM-02}, and \eqref{update-EM-03}.
In annealing schedules, the parameters $\beta$ and $\Gamma$ are changed from initial values to $1$ and $0$ via iterations, respectively.

\subsection{Numerical results} \label{sec-numerical-02}

To clarify the advantage of DQAEM, we compare DQAEM with EM and DSAEM by using numerical simulations. 
In the annealing schedule of DSAEM, we exponentially change $\beta$ from $0.7$ to $1$, following $\beta_t = (\beta_0 - 1) \exp ( - t / 0.95 ) + 1$ where $\beta_t$ denotes $\beta$ at the $t$-iteration.
On the other hand, in DQAEM, we exponentially change $\Gamma$ from $1.2$ to $0$, following $\Gamma_t = \Gamma_0 \exp (- t / 0.95 )$ where $\Gamma_t$ is $\Gamma$ at the $t$-th iteration.
To compare the ``pure" quantum effect in DQAEM with the thermal effect in DSAEM, we fix $\beta_t = 1.0$ in the annealing schedule of DQAEM.
In the numerical simulation, we apply three algorithms to the two-dimensional data set shown in Fig.~\ref{numerical-05-01}(a).

As a result, the log-likelihood functions of EM, the negative free energy functions of DSAEM and those of DQAEM are plotted by red lines, orange lines and blue lines in Fig.~\ref{numerical-05-01}(b), respectively.
In this numerical simulation, the optimal value of the log likelihood function is $-3616.6$, which is depicted by the green line.
From this figure, we see that some trials of three algorithms succeed to find the global optimum, but others fail.
\begin{figure}[t]
\centering
\includegraphics[scale=0.55]{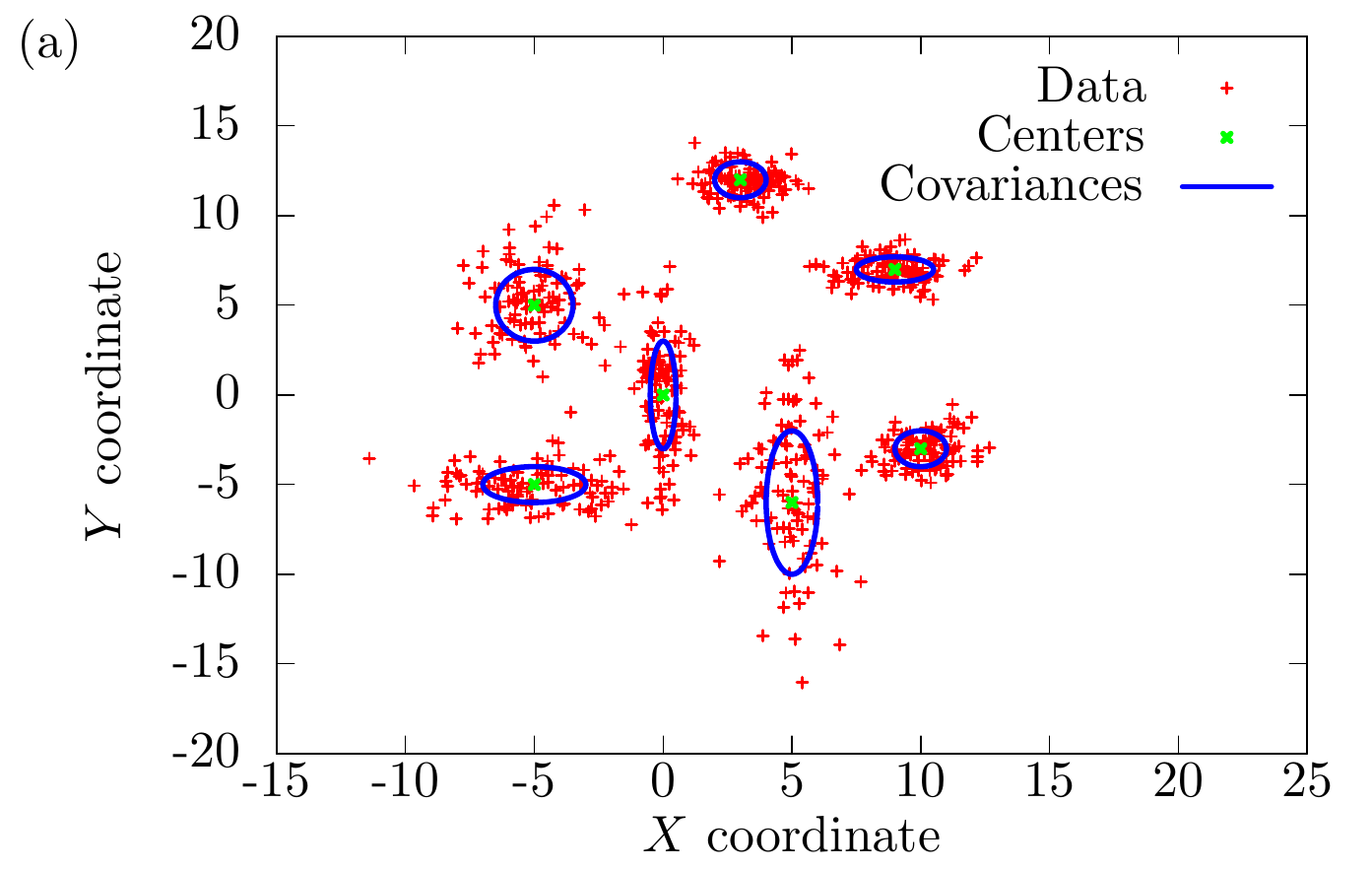}
\includegraphics[scale=0.55]{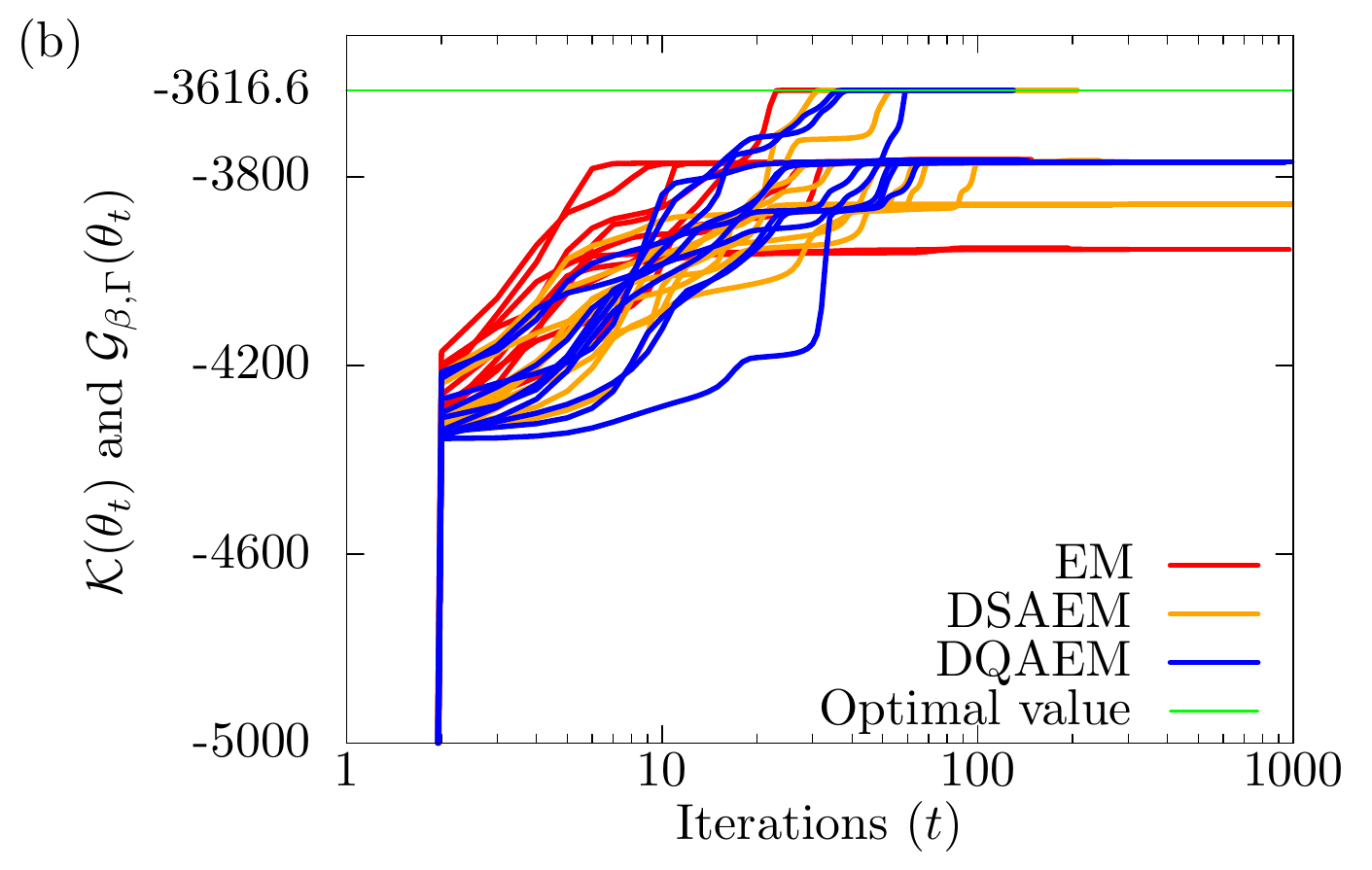}
\caption{(a) Data set generated by seven Gaussian functions. Green x-marks and blue lines depict centers and covariances of Gaussian functions, respectively. (b) Number of iterations (log scale) vs $\mathcal{K}(\theta)$ in Eq.~\eqref{log-likelihood-xx-01} and $\mathcal{G}_{\beta, \Gamma}(\theta)$ in Eq.~\eqref{free-energy-66}.}
\label{numerical-05-01}
\end{figure}

To clarify the performances of three algorithms, we perform DQAEM, EM and DSAEM 1000 times, respectively.
We summarize the success ratios of DQAEM, EM and DSAEM in Table~\ref{comparison-of-three-01-xx}.
From Table.~\ref{comparison-of-three-01-xx}, we conclude that DQAEM is superior to both EM and DSAEM.
\begin{table}
\caption{Success ratios of DQAEM, EM, and DSAEM.}
\label{comparison-of-three-01-xx}
\begin{tabular*}{\textwidth}{@{}l*{15}{@{\extracolsep{0pt plus12pt}}l}}
\br
DQAEM       & EM &          DSAEM   \\ \hline
40.5 \%    &  20.3 \% &    34.3 \% \\
\br
\end{tabular*}
\end{table}
For intuitive understanding of the reason why DQAEM outperforms EM and DSAEM, we add some more numerical simulations in~\ref{sec-numerical-03}.

\section{Conclusion} \label{conc}

In this paper, we have proposed a new algorithm, which we call DQAEM, to relax the problem of local optima of EM and DSAEM by introducing the mechanism of quantum fluctuations.
After formulating it, we have discussed its mechanism from the viewpoint of the path integral formulation.
Furthermore, as applications, we have adopted GMMs.
Then we have elucidated that DQAEM outperforms EM and DSAEM through numerical simulations.
Before closing this paper, we note that the form of the non-commutative term $H^\mathrm{nc}$ has arbitrariness.
The optimal choice of the term is an open problem similarly to original quantum annealing.

\section*{Acknowledgements} \label{ack-00}

HM thanks to Masayuki Ohzeki and Shu Tanaka for fruitful discussions.
This research is partially supported by the Platform for Dynamic Approaches to Living Systems funded by MEXT and AMED, and Grant-in-Aid for Scientific Research (B) (16H04382), Japan Society for the Promotion of Science.

\bibliographystyle{iopart-num}
\bibliography{paper-dqaem-spin-99-01-bib}

\appendix

\section{Convergence theorem} \label{app-dqaem-gmm-02}

We give a theorem that the negative free energy function has monotonicity for iterations.
\begin{theorem} \label{theorem01}
Let $\theta_{t+1} = \argmax_\theta \, \mathcal{U}_{\beta, \Gamma}(\theta;\theta_{t})$.
Then $\mathcal{G}_{\beta, \Gamma}(\theta_{t+1}) \ge \mathcal{G}_{\beta, \Gamma}(\theta_{t})$ holds.
\end{theorem}

\begin{proof}
First, the difference of the free energy functions in each iteration can be written as
\begin{align}
\ifthenelse{\versionmath = 1}
{
\Delta \mathcal{G}_{\beta, \Gamma} &= \mathcal{G}_{\beta,\Gamma}(\theta_{t+1}) - \mathcal{G}_{\beta,\Gamma} (\theta_{t}) \\
								   &= \mathcal{U}_{\beta, \Gamma}(\theta_{t+1};\theta_{t}) - \mathcal{U}_{\beta, \Gamma} (\theta_{t};\theta_{t}) - \frac{1}{\beta} (\mathcal{S}_{\beta, \Gamma}(\theta_{t+1};\theta_{t}) - \mathcal{S}_{\beta, \Gamma} (\theta_{t};\theta_{t})), \label{QA-proof-01}
}
{
\Delta \mathcal{G}_{\beta, \Gamma} &= \mathcal{G}_{\beta,\Gamma}(\theta_{t+1}) - \mathcal{G}_{\beta,\Gamma} (\theta_{t}) \\
								   &= \frac{1}{\beta}(\mathcal{U}_{\beta, \Gamma}(\theta_{t+1};\theta_{t}) - \mathcal{U}_{\beta, \Gamma} (\theta_{t};\theta_{t})) - \frac{1}{\beta} (\mathcal{S}_{\beta, \Gamma}(\theta_{t+1};\theta_{t}) - \mathcal{S}_{\beta, \Gamma} (\theta_{t};\theta_{t})), \label{QA-proof-01}
}
\end{align}
where
\begin{align}
\mathcal{S}_{\beta,\Gamma} (\theta; \theta') &= \sum_{i=1}^N \mathrm{Tr}_{\sigma_{i}} \bigg[ \frac{f_{\beta, \Gamma} (y_{i}, \hat{\sigma}_{i}; \theta')}{\mathrm{Tr}_{\sigma_i} \left[ f_{\beta, \Gamma} (y_{i}, \hat{\sigma}_{i}; \theta') \right] } \ln \frac{f_{\beta, \Gamma} (y_{i}, \hat{\sigma}_{i}; \theta)}{\mathrm{Tr}_{\sigma_i} \left[ f_{\beta, \Gamma} (y_{i}, \hat{\sigma}_{i}; \theta) \right] } \bigg].
\end{align}
The first two terms in the right hand side of Eq.~\eqref{QA-proof-01} is positive due to Eq.~\eqref{QAEM-Mstep03}.
Furthermore, we can show that the rest of the right hand side of Eq.~\eqref{QA-proof-01} are positive by the following calculations:
\begin{align}
- (& \mathcal{S}_{\beta, \Gamma}(\theta_{t+1};\theta_{t}) - \mathcal{S}_{\beta, \Gamma} (\theta_{t};\theta_{t})) \nonumber \\
   &= \mathrm{KL} \left( \frac{f_{\beta, \Gamma} (y_{i}, \hat{\sigma}_{i}; \theta_{t+1})}{\mathrm{Tr}_{\sigma_i} \left[ f_{\beta, \Gamma} (y_{i}, \hat{\sigma}_{i}; \theta_{t+1}) \right] } \middle\| \frac{f_{\beta, \Gamma} (y_{i}, \hat{\sigma}_{i}; \theta_{t})}{\mathrm{Tr}_{\sigma_i} \left[ f_{\beta, \Gamma} (y_{i}, \hat{\sigma}_{i}; \theta_{t}) \right] } \right) \nonumber \\
   &\ge 0.
\end{align}
\end{proof}

This theorem insists that DQAEM converges to at least the global optimum or a local optimum.
The global convergence of EM is discussed by Dempster \textit{et al.}~\cite{Dempster01} and Wu~\cite{Wu01}, and their discussion is available to DQAEM.

\section{Numerical results} \label{sec-numerical-03}

We discuss the performances of DQAEM, EM, and DSAEM by showing how the landscape of the negative free energy function behaves when $\beta$ and $\Gamma$ are changed.
For simplicity, we assume that a GMM is composed of two one-dimensional Gaussian functions and consider the case where only means of two Gaussian functions are estimated.
Accordingly, the joint probability density function is given by
\begin{align}
f(y_i, \sigma_i; \theta = \{\mu^1, \mu^2\}) &= \prod_{k=1}^2 \pi^{k} g(y_i; \mu^k, \Sigma^k) \delta_{\sigma_i, k},
\end{align}
where $g(y_i; \mu, \Sigma)$ is a Gaussian function with mean $\mu$ and covariance $\Sigma$.
Under this setup, we estimate $\theta = \{\mu^1, \mu^2\}$: we assume that $\theta = \{\mu^1, \mu^2\}$ are unknown and $\{\pi^1, \pi^2, \Sigma^1, \Sigma^2\}$ are given.

First, let us describe the landscapes of the negative free energy functions with different $\beta$ and $\Gamma$.
We plot the negative free energy functions of DQAEM at $\Gamma = 0.0$, $5.0$, $10.0$, and $50.0$ with $\beta = 1$ in Fig.~\ref{numerical-01-02}(a), (b), (c), and (d), respectively.
Specifically, Fig.~\ref{numerical-01-02}(a) describes the log-likelihood function, since DQAEM reduces to EM at $\beta = 1$ and $\Gamma = 0$.
In our example used here, the global optimal value for $\{\mu_1, \mu_2\}$ is $\{-2.0, 4.0\}$ (the left top in Fig.~\ref{numerical-01-02}(a)) and the point $\{4.0, -2.0\}$ (the right top in Fig.~\ref{numerical-01-02}(a)) is the local optimum.
If we set the initial value at a point close to the local optimum, EM fails to find the global optimum.
In contrast, Fig.~\ref{numerical-01-02} shows that the negative free energy function changes from a multimodal form to a unimodal form when $\Gamma$ increases.
Thus, even if we set the initial value close to the local optimum, DQAEM may find the global optimum through the annealing process.
On the other hand, in Fig.~\ref{numerical-01-01}, we plot the negative free energy functions of DSAEM at inverse temperature $\beta = 1.0$, $0.3$, $0.1$, and $0.01$.
DSAEM has the similar effect in the negative free energy function.

\begin{figure}[t]
\centering
\includegraphics[scale=0.3]{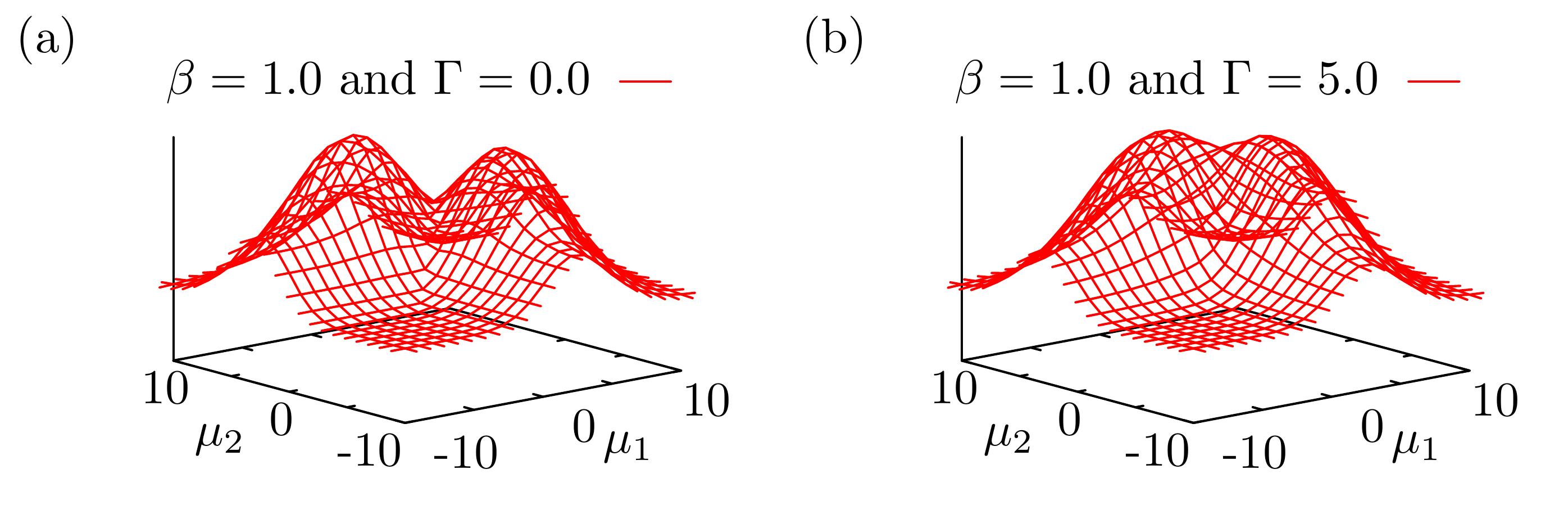}
\includegraphics[scale=0.3]{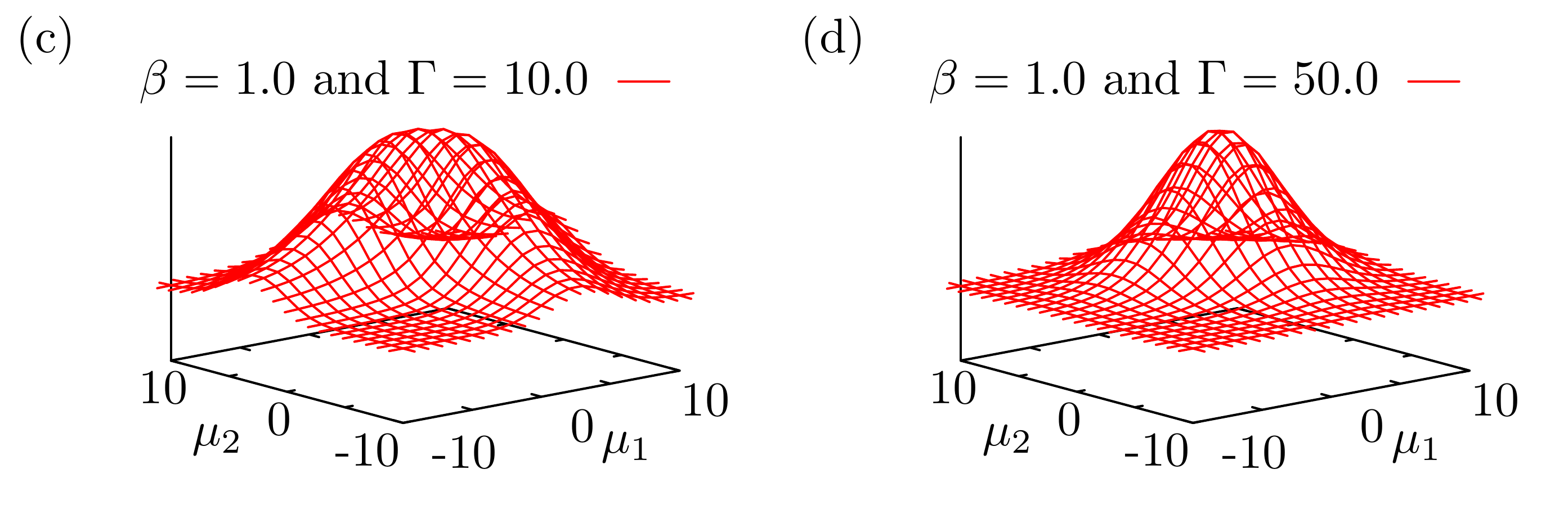}
\caption{Negative free energy functions for (a) $\Gamma = 0.0$, (b) $\Gamma = 5.0$, (c) $\Gamma = 10.0$, and (d) $\Gamma = 50.0$. $\beta$ is set to $1.0$.
 }
\label{numerical-01-02}
\end{figure}

However, as we have seen in Sec.~\ref{sec-numerical-02}, they differ qualitatively.
To understand this fact intuitively, we show the trajectories of the estimated parameters of EM, DSAEM, and DQAEM in Fig.~\ref{numerical-02-01}.
In this numerical simulation, the initial parameter $\theta_{0}$ is set near the local optimum in the log-likelihood function. 
The red line depicts the trajectory of the estimated parameter of EM and goes to the local optimum orthogonally crossing the contour plots.
As shown by the orange line, DSAEM eventually fails to find the global optimum as same as the case of EM.
On the other hand, the estimated parameter $\theta_{t}$ of DQAEM shown by the blue line surmounts the potential barrier and reaches the global optimum.
Accordingly, we consider that DQAEM outperforms EM and DSAEM.

\begin{figure}[t]
\centering
\includegraphics[scale=0.3]{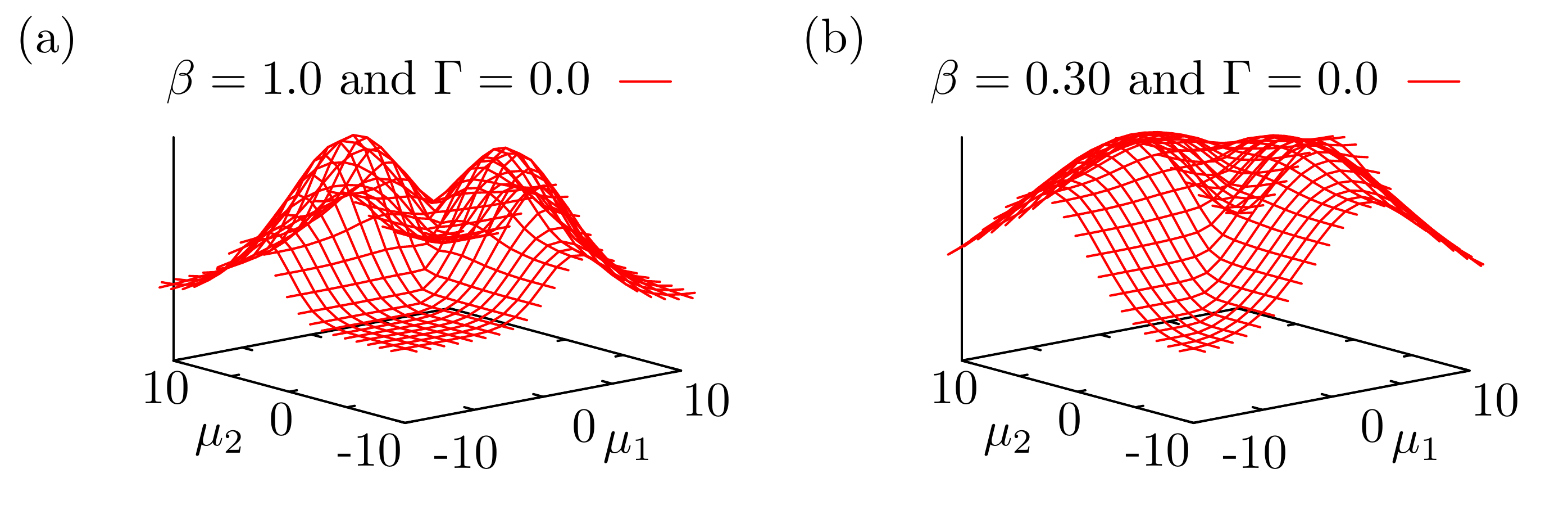}
\includegraphics[scale=0.3]{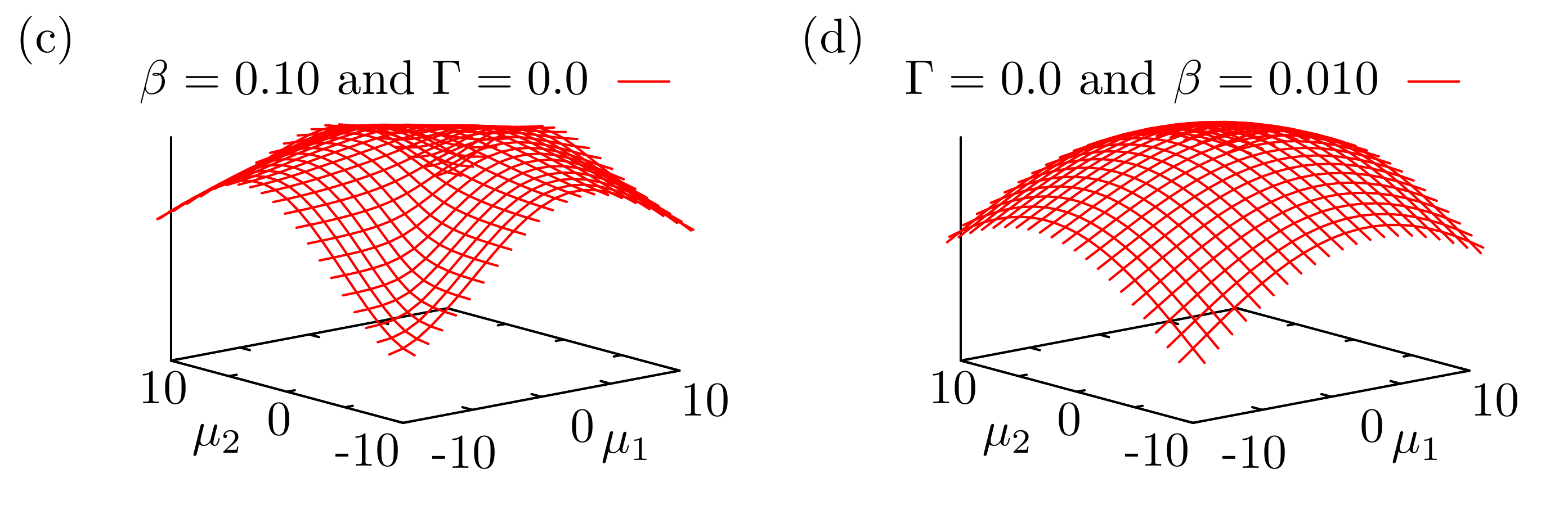}
\caption{Negative free energy functions for (a) $\beta = 1.0$, (b) $\beta = 0.3$, (c) $\beta = 0.1$, and (d) $\beta = 0.01$. $\Gamma$ is set to $0$.
}
\label{numerical-01-01}
\end{figure}

\begin{figure}[t]
\centering
\includegraphics[scale=0.3]{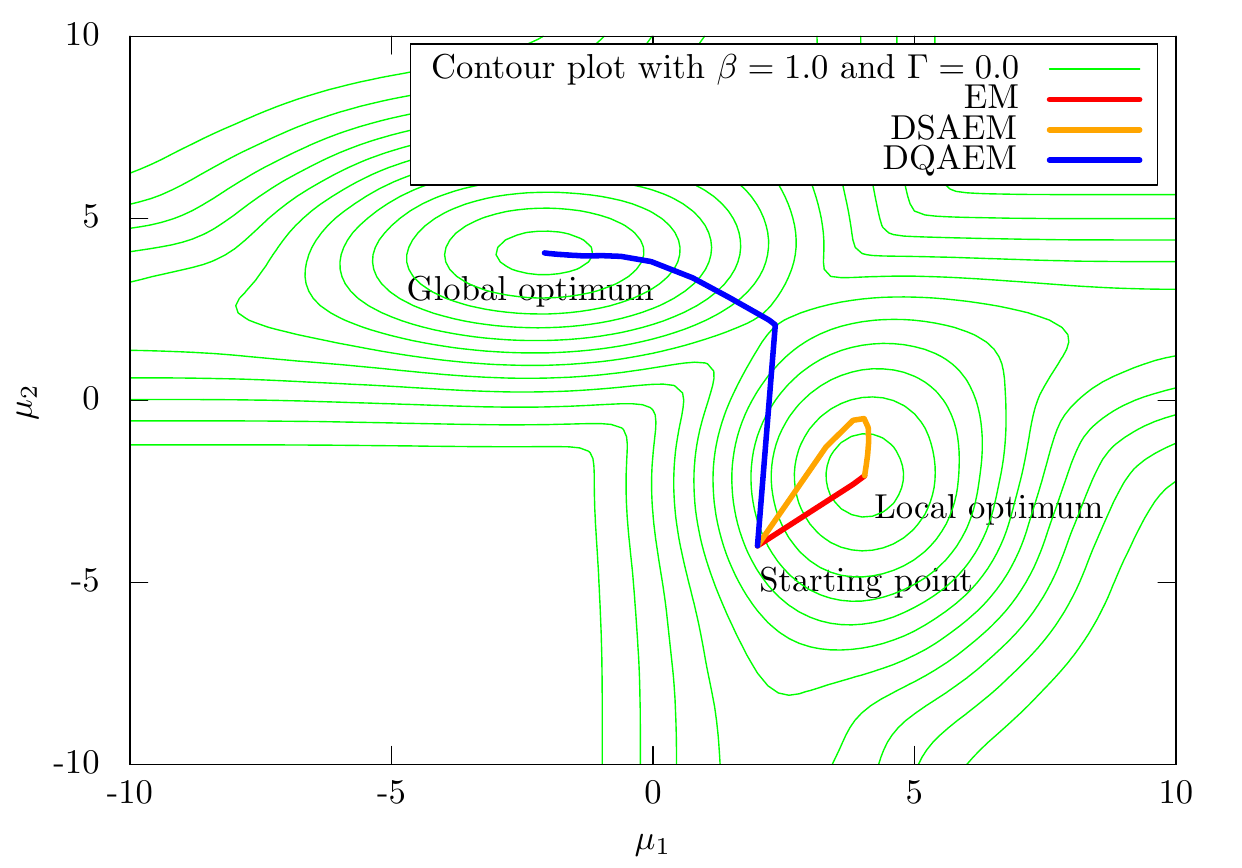}
\caption{Trajectories of the estimated parameters of EM (the red line), DSAEM (the orange line), and DQAEM (the blue line). The initial input is $\{\mu_0^1, \mu_0^2\} = \{2.0, -4.0\}$ and the annealing schedule for DQAEM and DSAEM are fixed. The green lines represent the contour plot of the log-likelihood function.}
\label{numerical-02-01}
\end{figure}

\end{document}